%% file: nips-modelsteal.tex
\documentclass{article}

\usepackage[preprint]{neurips_2026}

\usepackage[utf8]{inputenc} 
\usepackage[T1]{fontenc}    
\usepackage{hyperref}       
\usepackage{url}            
\usepackage{booktabs}       
\usepackage{amsfonts}       
\usepackage{nicefrac}       
\usepackage{microtype}      
\usepackage{xcolor}         

\usepackage{amsmath}
\usepackage{amssymb}
\usepackage{amsthm}
\usepackage{graphicx}
\usepackage{subcaption}
\usepackage{algorithm}
\usepackage{algorithmic}
\usepackage{mathtools}
\usepackage{cleveref}

\newtheorem{theorem}{Theorem}
\newtheorem{proposition}[theorem]{Proposition}

\newtheorem{corollary}[theorem]{Corollary}
\newtheorem{definition}[theorem]{Definition}
\newtheorem{remark}[theorem]{Remark}

\input{math_commands}

\title{Identified-Set Geometry of Distributional Model Extraction\\under Top-$K$ Censored API Access}

\author{%
\textbf{Wenhua Nie \quad ZiCheng Zhu \quad Jianan Wu}\\
\textbf{BINHAN LUO \quad Haoran Zheng \quad Jyh-Shing Roger Jang}\\[3pt]
\normalfont Correspondence: Wenhua Nie, National Taiwan University\\
\normalfont \texttt{d13944014@ntu.edu.tw}
}

\begin{document}

\maketitle

\input{sections/0_abstract}
\input{sections/1_introduction}
\input{sections/2_related_work}
\input{sections/3_identified_set}
\input{sections/4_minimax}
\input{sections/5_experiments}
\input{sections/6_conclusion}

\clearpage
\bibliographystyle{plainnat}
\bibliography{references}

\appendix
\input{sections/A_appendix}

\end{document}

%% file: math_commands.tex
\newcommand{\R}{\mathbb{R}}

\newcommand{\cS}{\mathcal{S}}

\newcommand{\bz}{\mathbf{z}}
\newcommand{\bp}{\mathbf{p}}
\newcommand{\bq}{\mathbf{q}}
\newcommand{\bx}{\mathbf{x}}

\DeclareMathOperator{\TV}{TV}
\DeclareMathOperator{\KL}{KL}
\DeclareMathOperator{\diam}{diam}

\DeclareMathOperator{\softmax}{softmax}

\newcommand{\Uk}{U_K}
\newcommand{\Ur}{U_R}
\newcommand{\ZA}{Z_A}

%% file: sections/0_abstract.tex
\begin{abstract}
Modern LLM APIs often reveal only top-$K$ logit scores and censor the remaining vocabulary.
We study the per-position distribution-recovery limits of this access model.
For censoring threshold~$\tau$, the compatible teacher distributions form an identified set whose total-variation diameter is exactly
$\Uk=(V{-}K)\exp(\tau)/(Z_A+(V{-}K)\exp(\tau)),$
where $Z_A$ is the observed partition function.
For KL recovery, we give a computable binary-endpoint lower bound and an asymptotically matching small-ambiguity upper bound, with an extension to reference-aware attackers.
Experiments on a Qwen3 math-reasoning teacher reveal a layered extraction hierarchy: on-task top-$K$ distillation recovers 12\% of private capability, full-logit distillation recovers 56\% despite 99\% KL closure, and generation-based extraction recovers 96\%.
Top-$K$ censoring therefore limits per-position distribution recovery but does not by itself prevent capability extraction, separating fidelity from transfer in prompt-only logit distillation.
\end{abstract}

%% file: sections/1_introduction.tex
\section{Introduction}
\label{sec:intro}

Large language models deployed behind APIs represent significant investments in data curation, training compute, and algorithmic innovation. The security assumption underlying this deployment model is that restricting access to top-$K$ scores or log-probabilities---rather than full vocabulary distributions---provides meaningful protection against model extraction. But how much protection does top-$K$ censoring actually provide?

Recent work has begun to quantify the extraction threat. \citet{carlini2024stealing} demonstrated algebraic extraction of the output projection layer from production APIs. \citet{golowich2025provably} proved that model extraction is polynomial-time efficient when the learner has access to \emph{full} conditional probability vectors and the model has low logit rank. However, these results assume access that most real APIs do not provide: production endpoints typically return only the top-$K$ log-probabilities (e.g., $K{=}5$ or $K{=}20$), censoring the remaining $V{-}K$ entries entirely.

\paragraph{This paper.} To our knowledge, we give the first exact identified-set characterization for the per-position distribution-recovery component of top-$K$ censored model extraction, together with a KL lower bound and small-ambiguity asymptotics. Our framework treats the censored observation as a \emph{partial identification} problem~\citep{manski2003partial}: for each prompt, the top-$K$ scores define an \emph{identified set} of compatible teacher distributions, and the attacker's task is to recover the true distribution from within this set. The main theorem analyzes the harder relative-score case where the normalizer is unknown; when an API returns calibrated normalized log-probabilities, the ambiguity shrinks to the observed hidden tail mass (\Cref{cor:normalized}). Capability extraction is evaluated separately because downstream skill transfer need not be monotone in per-position KL.

\paragraph{Contributions.}
\begin{enumerate}
    \item \textbf{Exact identified-set geometry} (\S\ref{sec:identified-set}). Under the \emph{unnormalized-logit} access model (the harder case, where the normalizing constant is unknown), we prove that the closure of teacher distributions compatible with a top-$K$ observation has total-variation diameter exactly $\Uk = (V{-}K)\exp(\tau)/(Z_A + (V{-}K)\exp(\tau))$, where $\tau$ is the $K$-th logit and $Z_A$ is the observed partition function. This quantity is computable from a single API call. We also characterize the strictly easier normalized-log-probability case as a corollary (\Cref{cor:normalized}).

    \item \textbf{KL recovery limits} (\S\ref{sec:minimax}). For the unnormalized access model, we prove a computable binary-endpoint lower bound $R_{\mathrm{bin}}(\Uk) \leq R_K^*$ and an asymptotically matching upper bound $\Uk/e + O(\Uk^2)$ in the small-ambiguity regime. For the large $\Uk$ values observed in realistic top-$K$ APIs, $R_{\mathrm{bin}}$ is a conservative certified lower bound; interior concentrated-tail adversaries can make the full finite-$\Uk$ minimax risk larger.

    \item \textbf{Layered extraction hierarchy} (\S\ref{sec:experiments}). We decompose capability extraction into distribution and access-mode effects. On math reasoning: on-task top-$K$ KD recovers 12\% PVR, on-task full-logit KD recovers 56\%, and generation access recovers 96\%---an 84-point span across access modes on the same prompt distribution. Off-task methods expose the fidelity-transfer paradox: sparse top-$K$ KD reaches 28\% PVR despite negative KL closure. On code generation, prompt-only logit methods recover 44--46\% PVR while generation reaches teacher-level execution.

    \item \textbf{$K$-sweep and cross-family stress tests} (\S\ref{sec:experiments}). We sweep $K \in \{1, 5, 10, 20, 50, 100\}$ on Qwen3-0.6B and show extraction error tracks the theoretical $\Uk$ prediction. The $\Uk$ calculation is architecture-agnostic: on Llama-3.2-1B ($V{=}128{,}256$), $\Uk = 0.942$ at $K{=}20$. We also repeat WikiText distribution-recovery extraction on Llama-3.2-1B/3B teachers, where top-$K$ \texttt{delta\_rank} closes 54\%/40\% of CE-to-teacher KL.
\end{enumerate}

The central message is not that top-$K$ censoring is a complete extraction defense. Rather, it precisely characterizes one API surface: prompt-only, per-position logit access. Generation and autoregressive logit access are deliberately treated as stronger oracles; their high PVR shows that logit censoring complements but cannot replace endpoint controls. This also separates our contribution from generation-distillation work: we isolate the information left by censored logits, then measure capability transfer as a distinct outcome.

%% file: sections/2_related_work.tex
\section{Related Work}
\label{sec:related}

\paragraph{Model extraction attacks.}
Model stealing has progressed from equation-solving attacks on linear models~\citep{tramer2016stealing} to functionally equivalent extraction of deep networks~\citep{jagielski2020high}. For LLMs, \citet{carlini2024stealing} achieved a qualitative leap by algebraically recovering the output projection matrix from production API queries, exploiting the affine structure of the last-layer mapping. \citet{finlayson2024logits} showed that even limited API access leaks architectural information such as hidden dimension. However, these attacks are \emph{operational}---they demonstrate what specific algorithms can extract---without characterizing the fundamental information-theoretic limits of the access model. Our work fills this gap: we prove what \emph{no} algorithm can recover from a single top-$K$ censored observation per prompt position, regardless of computational budget.

\paragraph{Theoretical model extraction.}
\citet{golowich2025sequences} showed that the log-probability matrix of autoregressive LMs has approximate low rank, and \citet{golowich2025provably} proved that a polynomial-time algorithm can learn any such model from conditional probability queries. These results establish that model extraction is efficient \emph{when the learner sees full logit vectors}. Our contribution is orthogonal: we ask what happens when the logit vector is censored to its top-$K$ entries. The identified set framework shows that top-$K$ censoring creates an irreducible \emph{per-position} ambiguity floor: for each prompt-position pair, the top-$K$ observation leaves TV diameter $\Uk$ fundamentally unresolved.

\paragraph{Knowledge distillation under access constraints.}
Knowledge distillation trains a student by matching teacher predictions~\citep{hinton2015distilling}; recent LLM distillation work has increasingly addressed restricted access settings. \citet{agarwal2024gkd} (GKD) identified the train-inference distribution mismatch of standard logit-matching KD and proposed on-policy generation as a fix. \citet{gu2024minillm} showed that forward KL causes overestimation in generative LLMs and switched to reverse KL with on-policy samples. DeepSeek-R1~\citep{deepseek2025r1} demonstrated that reasoning capability transfers via SFT on model-generated chain-of-thought traces, which does not require logit access. These works establish empirically that generation-based distillation outperforms logit-based distillation for capability transfer, but they do not quantify the access-mode decomposition from top-$K$ logits to full logits to generation, nor the non-monotonic relationship between KL closure and private-value recovery. Our contribution is not a new distillation method; rather, we characterize a complementary distributional limit: the identified set under top-$K$ censoring has TV diameter $\Uk$ for each prompt-position distribution, while downstream capability transfer must be measured separately.

\paragraph{Partial identification.}
Partial identification is a well-established framework in econometrics and statistics for inference under incomplete data~\citep{manski2003partial,molinari2020microeconometrics}. When the data do not uniquely identify a parameter, the goal shifts from point estimation to characterizing the \emph{identified set}---the family of parameter values consistent with observations. Despite its generality, partial identification has not previously been applied to LLM APIs. We adapt this framework to the top-$K$ censored logit setting, deriving the exact geometry of the identified set and computable KL recovery lower bounds. The closest mathematical precedent is censored regression (Tobit models), but our setting differs fundamentally: we observe an order-statistic truncation of a softmax distribution, not a location-censored linear model.

\paragraph{Top-$K$ in inference and training.}
The top-$K$ truncation has been studied for inference quality~\citep{hewitt2022truncation} and efficient attention~\citep{topk_attention_tv2025}. \citet{hewitt2022truncation} views top-$K$ sampling as recovering an unsmoothed support; \citet{topk_attention_tv2025} derives the TV distance between full and truncated attention distributions. Our work differs in the object of study: these papers analyze approximation error $\TV(P, \hat{P})$ for a \emph{fixed} truncation, while we characterize the \emph{diameter} $\diam_{\TV}(\bar{\cS}_K)$ of the \emph{identified set}---the worst-case ambiguity over \emph{all} distributions compatible with the observation. These are fundamentally different quantities answering different questions.

%% file: sections/3_identified_set.tex
\section{The Identified Set under Top-$K$ Censoring}
\label{sec:identified-set}

\subsection{Top-$K$ Observation Model}

Consider a teacher LLM with vocabulary $[V] = \{1, \ldots, V\}$. For a given prompt $\bx$, the teacher produces a next-token logit vector $\bz \in \R^V$ and a distribution $\bp = \softmax(\bz)$.

\paragraph{Access model.} We consider a top-$K$ API that returns \emph{unnormalized logits} (equivalently, log-probabilities up to an unknown additive shift) for the $K$ highest-scoring tokens:
\begin{enumerate}
    \item The index set $A \subset [V]$ of the $K$ highest-probability tokens ($|A| = K$);
    \item The raw logit values $\{z_v\}_{v \in A}$ (not normalized log-probabilities);
    \item Implicitly, the censoring threshold $\tau = z_{(K)}$ (the $K$-th largest logit).
\end{enumerate}
All remaining logits are hidden, subject only to $z_u \leq \tau$. This models APIs that reveal top-$K$ relative scores while hiding the absolute normalizer $\log \sum_v \exp(z_v)$.

\paragraph{Normalized log-probability APIs.} If the API instead returns \emph{normalized} log-probabilities $\log p_v$ for $v \in A$, then the observed head mass $\sum_{v \in A} p_v$ is known exactly, and the hidden tail mass is $1 - \sum_{v \in A} p_v$. In this case, the TV diameter of the identified set reduces to precisely this hidden tail mass. Our $\Uk$ formula remains valid as an upper bound (since the hidden tail mass $\leq \Uk$), and equals it when the normalizer is unknown. We present the unnormalized-logit setting as the primary formulation because it is the harder (more ambiguous) case and subsumes the normalized case.

\begin{definition}[Identified set]
\label{def:identified-set}
The \emph{identified set} under a top-$K$ observation $(A, \{z_v\}_{v \in A}, \tau)$ is:
\begin{equation}
\label{eq:identified-set}
\cS_K = \bigl\{ \bp = \softmax(\bz') : z'_v = z_v \;\forall v \in A,\; z'_u \leq \tau \;\forall u \notin A \bigr\}.
\end{equation}
Its closure $\bar{\cS}_K$ admits tail logits $z'_u \to -\infty$ (zero probability on censored tokens).
\end{definition}

\subsection{Probability-Space Representation}

Let $M = V - K$, $b = \exp(\tau)$, and $\ZA = \sum_{v \in A} \exp(z_v)$. Define the conditional head distribution $\alpha_v = \exp(z_v) / \ZA$ for $v \in A$. Every distribution in $\bar{\cS}_K$ has the form:
\begin{equation}
\label{eq:feasible}
p_v = (1-t)\alpha_v \;\text{ for } v \in A, \qquad \sum_{u \notin A} p_u = t, \qquad 0 \leq t \leq \Uk,
\end{equation}
where $t$ is the total tail mass. Each censored token satisfies the \emph{per-token cap}:
\begin{equation}
\label{eq:cap}
p_u \leq \frac{(1-t)\,\Uk}{M(1-\Uk)}, \qquad u \notin A.
\end{equation}
This cap arises from the logit ceiling $z_u \leq \tau$: the unnormalized weight $\exp(z_u) \leq b$ constrains each censored token's probability relative to the observed partition function.

\subsection{Exact TV Diameter}

\begin{theorem}[Identified set TV diameter]
\label{thm:tv-diameter}
\begin{equation}
\label{eq:uk}
\diam_{\TV}(\bar{\cS}_K) = \Uk = \frac{M \cdot \exp(\tau)}{\ZA + M \cdot \exp(\tau)}.
\end{equation}
\end{theorem}
\begin{proof}
For any two feasible distributions with tail masses $t,s$, the head conditional distribution is fixed at $\alpha$, giving
$\TV(\bp,\bq)\leq \tfrac12(|t-s|+t+s)=\max(t,s)\leq\Uk$.
Equality is achieved by the zero-tail distribution and the maximal uniform-tail distribution. Full details are in \Cref{app:proofs}.
\end{proof}

\begin{remark}[$\Uk$ is computable from a single API call]
\label{rem:computable}
The attacker can compute $\Uk$ directly from the top-$K$ observation: $\ZA$ requires only the $K$ revealed logits, $\tau$ is the smallest revealed logit, and $M = V - K$ requires knowing the vocabulary size (a public architectural detail). Thus the attacker can quantify their own distributional uncertainty without any additional queries.
\end{remark}

\begin{remark}[Scaling behavior]
As $K$ increases, $\Uk$ decreases monotonically: more revealed logits shrink the identified set. At $K = V$ (full logit access), $\Uk = 0$ and the distribution is exactly identified. At $K = 1$, $\Uk$ is close to 1 for typical LLM distributions, leaving nearly all probability mass ambiguous.
\end{remark}

\subsection{Normalized Log-Probability APIs}
\label{sec:normalized}

Many production APIs return \emph{normalized} log-probabilities $\ell_v = \log p_v$ for $v \in A$. Then the observed head mass $P_A$ is known and the hidden tail mass is $t^* = 1-P_A$.

\begin{corollary}[TV diameter under normalized access]
\label{cor:normalized}
If the API reveals normalized log-probabilities for the top-$K$ tokens, the hidden tail mass $t^* = 1 - \sum_{v \in A} p_v$ is known exactly. Each censored token satisfies $p_u \leq c = p_{(K)}$. Let $M = V - K$. Then:
\begin{equation}
\label{eq:uk-norm}
\diam_{\TV}(\bar{\cS}_K^{\,\mathrm{norm}}) \leq t^* \leq \Uk.
\end{equation}
Equality $\diam_{\TV} = t^*$ holds whenever $M \geq 2\lceil t^*/c \rceil$ (two allocations with disjoint supports exist). When $M = 1$, the diameter is $0$. For typical LLM vocabularies ($M > 10^5$, $c \ll t^*$), the cap is non-binding and $\diam_{\TV} = t^*$.
\end{corollary}

%% file: sections/4_minimax.tex
\section{Minimax Recovery and Reference-Aware Extension}
\label{sec:minimax}

\subsection{Minimax KL Recovery}

Given a top-$K$ observation, the attacker seeks an estimator $\bq$ that minimizes the worst-case KL divergence over the identified set:
\begin{equation}
\label{eq:minimax}
R_K^* = \inf_{\bq} \sup_{\bp \in \bar{\cS}_K} \KL(\bp \| \bq).
\end{equation}

\begin{theorem}[Asymptotic minimax KL bounds]
\label{thm:minimax}
As $\Uk \to 0$,
\begin{equation}
\label{eq:minimax-rate}
R_{\mathrm{bin}}(\Uk) \;\leq\; R_K^* \;\leq\; \frac{\Uk}{e} + O(\Uk^2),
\end{equation}
where $R_{\mathrm{bin}}(\Uk) = \Uk/e + O(\Uk^2)$ is the binary endpoint minimax lower bound (\Cref{rem:exact}). Thus the first-order minimax risk is $\Uk/e$ in the small-ambiguity regime. For the large $\Uk$ values observed in practice, $R_{\mathrm{bin}}$ is a conservative computable lower bound rather than an exact finite-$\Uk$ minimax value.
\end{theorem}

\begin{proof}[Proof sketch (full proof in \Cref{app:proofs})]
\textbf{Lower bound.} Restrict to the two extremes $\bp^{(0)}$ (zero tail) and $\bp^{(\Uk)}$ (maximal uniform tail). For any $\bq$ with tail mass $s$:
\begin{align}
\KL(\bp^{(0)} \| \bq) &= -\log(1-s), \label{eq:kl-min} \\
\KL(\bp^{(\Uk)} \| \bq) &= (1{-}\Uk)\log\frac{1{-}\Uk}{1{-}s} + \Uk\log\frac{\Uk}{s}. \label{eq:kl-max}
\end{align}
Balancing \eqref{eq:kl-min} and \eqref{eq:kl-max} gives $s^* = \Uk/e + O(\Uk^2)$ and $R_K^* \geq R_{\mathrm{bin}}(\Uk)=\Uk/e + O(\Uk^2)$.

\textbf{Upper bound.} Use the symmetric estimator with head mass $1-s^*$ and uniform tail mass $s^*/M$. The cap in \eqref{eq:cap} bounds the conditional tail KL, so the worst-case risk is at most $\Uk/e+O(\Uk^2)$ uniformly over $t\in[0,\Uk]$.
\end{proof}

\begin{remark}[Finite-$\Uk$ lower bound]
\label{rem:exact}
The binary endpoint reserve is $s^* = A_{\Uk}/(1+A_{\Uk})$ with $A_{\Uk} = \Uk(1{-}\Uk)^{(1-\Uk)/\Uk}$ and risk $R_{\mathrm{bin}}(\Uk) = -\log(1-s^*)$. Since observed $\Uk$ values are large, experiments report this computable lower bound rather than the asymptotic $\Uk/e$ approximation. The full finite-$\Uk$ minimax risk can be larger because an interior tail-mass adversary may concentrate probability up to the per-token cap; therefore $R_{\mathrm{bin}}$ should be read as a certified impossibility lower bound, not as an exact finite-$\Uk$ risk. \Cref{tab:gap} reports the finite-$\Uk$ gap for the proof's upper envelope.
\end{remark}

We emphasize that at the practical operating points measured in this paper ($\Uk>0.8$), the exact finite-$\Uk$ KL minimax rate remains open; the tight finite-regime statement is the TV diameter in \Cref{thm:tv-diameter}, while the KL result supplies a conservative certified lower bound and a first-order small-$\Uk$ characterization.

\begin{corollary}[Critical $K$]
\label{cor:kstar}
For a target KL recovery tolerance $\delta > 0$, the binary lower bound gives a necessary condition: if $R_{\mathrm{bin}}(\Uk) > \delta$, then no estimator can guarantee KL risk at most $\delta$ over the identified set. In the small-$\Uk$ regime this yields the first-order heuristic $\Uk \lesssim e\delta$.
\end{corollary}

\subsection{Reference-Aware Ambiguity}
\label{sec:reference}

When the attacker knows the teacher's base model, structural assumptions can tighten the identified set. We use a \emph{bounded perturbation} prior:

\begin{definition}[Bounded fine-tuning perturbation]
\label{def:bounded-ft}
The teacher is obtained by fine-tuning a known reference model with logits $z_{\mathrm{ref}}$. For each censored token $u \notin A$, the fine-tuning perturbation is bounded: $z_T(u) - z_{\mathrm{ref}}(u) \leq \rho$ for a known calibration margin $\rho > 0$. Combined with the top-$K$ constraint $z_T(u) \leq \tau$, this gives $z_T(u) \leq \min(\tau, z_{\mathrm{ref}}(u) + \rho)$.
\end{definition}

On observed tokens $\rho$ is calibratable, but compliance on censored tokens is untestable without full logits; it is therefore a structural prior about fine-tuning magnitude. Under this assumption:

\begin{theorem}[Reference-aware TV diameter]
\label{thm:reference}
Define $B_u = \exp\bigl(\min(\tau, z_{\mathrm{ref}}(u) + \rho)\bigr)$ and $C_R = \sum_{u \notin A} B_u$. Then:
\begin{equation}
\label{eq:ur}
\diam_{\TV}(\bar{\cS}_K^R) = \Ur = \frac{C_R}{\ZA + C_R} \leq \Uk.
\end{equation}
The analogous binary endpoint lower bound and small-$\Ur$ upper bound scale as $\Ur/e + O(\Ur^2)$ when the estimator reserves censored mass proportional to $B_u$.
\end{theorem}

The shrinkage $\Ur \leq \Uk$ quantifies the \emph{conditional} value of reference-model knowledge. Its validity depends on fine-tuning intensity; our experiments show it fails for aggressive SFT but is appropriate for mild adaptation.

%% file: sections/5_experiments.tex
\section{Experiments}
\label{sec:experiments}

We validate the theoretical bounds on real transformer next-token distributions and then separately measure the practical extraction consequence: distribution fidelity decouples from capability transfer under top-$K$ censoring.

\subsection{Setup}

\paragraph{Models.} We use Qwen3-0.6B (28 layers, $d{=}1024$, $V{=}151{,}936$) as the primary model family. The teacher is created by SFT on the GSM8K training set (3000 steps, lr=$2{\times}10^{-5}$), producing a math-capable model with GSM8K accuracy 44.6\% (base: 4.0\%, gap: $+$40.6 absolute points). The student is initialized from the same base model. For a second-family distribution-recovery stress test, we also use Llama-3.2-1B/3B teachers privately fine-tuned on WikiText; those runs report KL/top-1 agreement rather than PVR.

\paragraph{Metrics.} We report: (i) $\Uk$ (theoretical TV diameter); (ii) KL divergence to teacher; (iii) \emph{KL closure} $= 1 - \KL(T \| S) / \KL(T \| B)$ where $B$ is the base model trained on the same data distribution (WikiText) without teacher signal (\texttt{ce\_only}); (iv) \emph{Private Value Recovery} (PVR) $= (U_{\text{task}}(S) - U_{\text{task}}(B)) / (U_{\text{task}}(T) - U_{\text{task}}(B))$ on GSM8K exact-match accuracy. PVR measures what fraction of the teacher's private capability the student has acquired. Reported error bars are population standard deviations over seeds or prompt-position samples.

\paragraph{Methods.} We compare 9 extraction methods and controls across three access regimes:
\begin{itemize}
    \item \textbf{No-teacher and top-$K$ logit access}: \texttt{ce\_only} (no teacher signal), \texttt{strict\_topk\_kd} (sparse KL on top-$K$), \texttt{bild} (delta-teacher MSE), \texttt{delta\_rank} (delta + ranking + censored constraints).
    \item \textbf{Access-mode controls}: \texttt{dkd\_full} on WikiText prompts, plus \texttt{ontask\_topk} and \texttt{ontask\_dkd} on GSM8K prompts.
    \item \textbf{Generation access}: \texttt{gen\_sft} (SFT on teacher-generated chain-of-thought) and \texttt{joint\_gen\_kd} (generation SFT + logit KD).
\end{itemize}
All experiments use 3 random seeds; we report means.

\subsection{$\Uk$ on Real Models}

\Cref{tab:uk} reports $\Uk$ statistics computed over 3{,}200 prompt-position pairs from the math teacher. Even at $K{=}100$, the mean $\Uk$ remains $0.809$, indicating that ${\sim}81\%$ of the total-variation diameter is unresolved. At realistic API settings ($K{=}20$), $\Uk = 0.908$---the vast majority of the teacher's distribution is ambiguous.
For calibrated normalized top-$K$ log-probability APIs (\Cref{cor:normalized}), the relevant diameter is exactly the observed hidden tail mass, which averages $0.176$ at $K{=}20$; the unnormalized-access $\Uk$ is therefore a conservative harder-case bound.

\begin{table}[t]
\centering
\caption{$\Uk$ statistics on Qwen3-0.6B math teacher ($V{=}151{,}936$, 3{,}200 prompts). $R_{\mathrm{bin}}$ is the binary endpoint KL lower bound (\Cref{rem:exact}), a lower bound on the full minimax $R_K^*$. Even at $K{=}100$, over 80\% of the TV diameter remains unresolved.}
\label{tab:uk}
\small
\begin{tabular}{@{}r cc l@{}}
\toprule
$K$ & $\Uk$ (mean$\pm$std) & $R_{\mathrm{bin}}$ & \\
\midrule
5   & $0.984 \pm 0.088$ & $0.652$ & \\
10  & $0.970 \pm 0.124$ & $0.626$ & \\
20  & $0.908 \pm 0.237$ & $0.538$ & typical API setting \\
50  & $0.858 \pm 0.288$ & $0.484$ & \\
100 & $0.809 \pm 0.317$ & $0.437$ & \\
\bottomrule
\end{tabular}
\end{table}

\subsection{$K$-Sweep: Extraction Error Tracks $\Uk$}

\begin{figure}[t]
\centering
\includegraphics[width=0.48\textwidth]{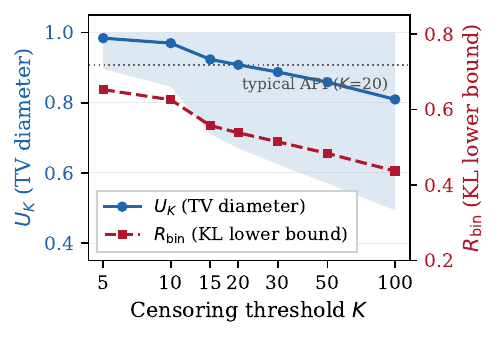}
\caption{TV diameter $\Uk$ and binary endpoint KL lower bound $R_{\mathrm{bin}}$ as a function of $K$ on Qwen3-0.6B math teacher (shading: $\pm 1$ population s.d.\ over 3{,}200 prompts). Even at $K{=}100$, $\Uk > 0.80$ and $R_{\mathrm{bin}} > 0.43$: the vast majority of the distribution remains unresolved.}
\label{fig:uk-sweep}
\end{figure}

We train \texttt{strict\_topk\_kd} and \texttt{delta\_rank} at each $K \in \{1, 5, 10, 20, 50, 100\}$ (3000 steps, 3 seeds). The empirical extraction error tracks the theoretical prediction: \texttt{strict\_topk\_kd}'s KL decreases monotonically from 3.15 ($K{=}1$) to 0.69 ($K{=}100$), paralleling the $\Uk$ curve. This is a training-averaged diagnostic: each SGD step observes many prompt-position pairs, while $\Uk$ bounds the ambiguity of each pair separately. \Cref{prop:nonadaptive-composition} gives the corresponding non-adaptive product-risk decomposition; formal adaptive composition remains open. \texttt{delta\_rank} achieves substantially lower KL than \texttt{strict\_topk\_kd} at $K{\geq}10$ but plateaus around 0.36--0.40, exhibiting a floor consistent with the logit KD ceiling. At $K{<}10$, \texttt{delta\_rank} is outside its intended operating range because the ranking/censoring constraints have too few candidate tokens; we report it only for $K{\geq}10$. The slight increase in \texttt{delta\_rank} KL from $K{=}10$ to $K{=}100$ (0.359$\to$0.397) is a diagnostic hypothesis: with more observed tokens, the ranking loss has more pairs to match, which can slightly conflict with the distributional objective. (All K-sweep results are regenerable from \texttt{results/k\_sweep/k\_sweep\_seed\{0,1,2\}.json}.)

\paragraph{Composition scope.}
The per-position identified set is the natural unit of partial identification. For independently drawn prompt positions and no cross-position structural assumptions, multiple observations produce a Cartesian product of per-position feasible sets rather than identifying the hidden tail of any individual position. Structure-exploiting adaptive attacks may do better: low-rank logit structure, autoregressive query-and-append access, or generation access can couple observations across positions and tighten or bypass the per-position barrier. Our experiments therefore test whether the average per-position ambiguity is predictive of prompt-only distillation error, not a formal adaptive-composition minimax theorem.

\subsection{Distribution Fidelity vs.\ Capability Transfer}

\Cref{tab:main} presents the central empirical result. On the GSM8K prompt distribution, top-$K$ KD recovers only 11.8\% PVR, full-logit DKD recovers 56.5\%, and generation SFT recovers 96.1\%. Off-task top-$K$ results additionally show that distributional fidelity substantially overstates private capability recovery: the best-KL \texttt{delta\_rank} variant achieves 60.0\% KL closure but only 7.2\% PVR, while sparse \texttt{strict\_topk\_kd} reaches 27.8\% PVR despite negative KL closure.

\begin{table}[t]
\centering
\caption{Layered extraction hierarchy on Qwen3-0.6B math teacher (GSM8K SFT, base 4.0\%$\to$teacher 44.6\%). Mean $\pm$ population s.d.\ over 3 seeds.}
\label{tab:main}
\small
\begin{tabular}{@{}l c l cc cc@{}}
\toprule
\textbf{Method} & \textbf{Access} & \textbf{Train data} & \textbf{KL}$\downarrow$ & \textbf{KL cl.} & \textbf{GSM8K} & \textbf{PVR} \\
\midrule
Base & --- & --- & --- & --- & 4.0\% & --- \\
Teacher & --- & --- & 0 & 100\% & 44.6\% & 100\% \\
\midrule
\texttt{ce\_only} & none & WikiText & $1.037_{\pm.002}$ & 0\% & $7.4_{\pm0.3}$ & $8.4_{\pm0.8}$ \\
\texttt{strict\_topk\_kd} & top-$K$ & WikiText & $1.183_{\pm.001}$ & $-$14\% & $15.3_{\pm1.7}$ & $27.8_{\pm4.1}$ \\
\texttt{bild} & top-$K$ & WikiText & $0.460_{\pm.002}$ & 55.6\% & $7.0_{\pm0.7}$ & $7.4_{\pm1.8}$ \\
\texttt{delta\_rank} & top-$K$ & WikiText & $\mathbf{0.414}_{\pm.002}$ & \textbf{60.0\%} & $6.9_{\pm1.2}$ & $7.2_{\pm2.9}$ \\
\midrule
\texttt{dkd\_full} & full logit & WikiText & $0.032_{\pm.000}$ & 96.9\% & $10.2_{\pm0.4}$ & $15.3_{\pm1.1}$ \\
\texttt{ontask\_topk} & top-$K$ & GSM8K$^\ddagger$ & $0.016_{\pm.000}$ & 98.4\% & $8.8_{\pm1.4}$ & $11.8_{\pm3.4}$ \\
\texttt{ontask\_dkd} & full logit & GSM8K$^\ddagger$ & $0.010_{\pm.000}$ & 99.0\% & $26.9_{\pm0.8}$ & $56.5_{\pm2.0}$ \\
\midrule
\texttt{gen\_sft} & generation & GSM8K & 0.48$^\dagger$ & --- & $43.0_{\pm1.0}$ & $\mathbf{96.1}_{\pm2.5}$ \\
\texttt{joint\_gen\_kd} & gen+logit & GSM8K+Wiki & 0.39$^\dagger$ & --- & $\mathbf{45.7}_{\pm1.1}$ & $\mathbf{102.6}_{\pm2.6}$ \\
\bottomrule
\end{tabular}
\end{table}

\noindent{\footnotesize $^\dagger$KL for generation-trained models is measured post-hoc on the same held-out evaluation distribution (WikiText validation), not during training.}

\noindent{\footnotesize $^\ddagger$On-task controls train on GSM8K prompts (same distribution as \texttt{gen\_sft}) without teacher-generated trajectories.}

\paragraph{Training distribution and access mode both matter.}
Full-logit KD on WikiText recovers only 15.3\% PVR; moving the same full-logit objective onto GSM8K prompts raises PVR to 56.5\%, so training distribution explains about 41 points. Holding the GSM8K prompt distribution fixed, replacing full logits with top-$K$ KD reduces PVR from 56.5\% to 11.8\%, isolating a 44.7-point access-mode drop for prompt-only KD. A further 39.6-point gap remains from on-task full-logit KD to generation SFT (96.1\%), which observes teacher trajectories rather than only prompt-position logits. Off-task results are method-dependent: \texttt{delta\_rank} yields 7.2\% PVR with strong KL closure, while \texttt{strict\_topk\_kd} reaches 27.8\% PVR while worsening KL. Top-$K$ censoring should therefore be read as a distribution-recovery barrier, not as a monotone PVR defense.
The on-task top-$K$ KD objective is a normalized observed-top-$K$ subdistribution objective, matching the normalized-access corollary rather than raw-logit scale matching. Its low PVR therefore shows that even the smaller normalized-access ambiguity does not by itself recover reasoning trajectories. The on-task KL closure values are measured on the same held-out WikiText distribution used by the CE denominator; they are sanity checks for distributional drift and should not be compared as strongly as the same-distribution PVR decomposition.

\paragraph{KL closure is not capability recovery.}
\texttt{strict\_topk\_kd} has negative KL closure ($-$14\%) yet 27.8\% PVR, while \texttt{dkd\_full} has 96.9\% closure but 15.3\% PVR. The decoupling is asymmetric: high fidelity does not imply high capability transfer, and low fidelity can still induce task gains through distributional bias.

A diagnostic explains why sparse top-$K$ KD can help despite worse KL: at $K{=}20$ in the K-sweep, \texttt{strict\_topk\_kd} has higher teacher top-1 agreement than \texttt{delta\_rank} (81.0\% vs.\ 75.9\%), so it better preserves the teacher's most likely token while distorting the long tail. PVR can therefore improve through top-token bias even when full-distribution KL worsens.

\paragraph{Generation and sequential logit access bypass prompt-only limits.}
\texttt{gen\_sft} recovers 96.1\% of math capability from 2{,}000 teacher-generated chain-of-thought answers, and \texttt{joint\_gen\_kd} reaches 102.6\%, statistically indistinguishable from full teacher recovery at the observed 2.6-point seed standard deviation. Sequential logit access (query, sample, append, repeat) is functionally equivalent to generation because it rolls out teacher trajectories token-by-token; our limitation is for \emph{prompt-only} single-pass logit KD.

\subsection{Supplementary Domain Check: The Decoupling is Task-Dependent}

We repeat extraction on a Qwen3-0.6B code teacher fine-tuned on MBPP (base 2.0\%$\to$teacher 29.0\% execution-pass). Code shows positive prompt-only transfer but a smaller access-mode separation: top-$K$ and full-logit objectives recover 46.3\% and 44.4\% PVR, respectively, while generation reaches teacher-level execution. The MBPP protocol gives each method one public assertion as a specification hint but excludes that assertion from scoring, so execution-pass is measured on held-out assertions. Results are mean $\pm$ population s.d.\ over 3 seeds; protocol details are in \Cref{app:proofs}.

\begin{table}[t]
\centering
\caption{Supplementary task-domain check (mean $\pm$ population s.d., 3 seeds). Math reasoning shows severe decoupling; code generation (execution-pass metric) shows positive prompt-only transfer but a smaller access-mode separation. PVR = fraction of teacher's private capability recovered.}
\label{tab:multidomain}
\small
\begin{tabular}{@{}l cc cc@{}}
\toprule
& \multicolumn{2}{c}{Math (GSM8K)} & \multicolumn{2}{c}{Code (MBPP exec)} \\
\cmidrule(lr){2-3} \cmidrule(lr){4-5}
\textbf{Method} & KL & PVR & KL & PVR \\
\midrule
\texttt{delta\_rank} / top-$K$ KL & $0.414$ & $7.2_{\pm2.9}$ & $0.019_{\pm.003}$ & $46.3_{\pm7.9}$ \\
\texttt{dkd\_full} (full) & $0.032$ & $15.3_{\pm1.1}$ & $0.019_{\pm.006}$ & $44.4_{\pm5.2}$ \\
\texttt{gen\_sft} & --- & $96.1_{\pm2.5}$ & --- & $100.0_{\pm4.0}$ \\
\bottomrule
\end{tabular}
\end{table}
For MBPP, the top-$K$ row is a normalized-log-probability API control: it renormalizes the observed top-$K$ distribution and compares that access regime to full logits, rather than reusing the math-domain residual ranking loss.

Thus $\Uk$ constrains distributional recovery, but capability transfer depends on how the private skill is encoded: math reasoning needs trajectories, while code generation shows more prompt-only transfer under this controlled MBPP protocol.

\subsection{Cross-Family $\Uk$ and Distribution Recovery}

The $\Uk$ geometry generalizes across architectures: on Llama-3.2-1B ($V{=}128{,}256$), $\Uk = 0.942$ at $K{=}20$ (vs.\ 0.908 for Qwen3), and both stay above 0.80 at $K{=}100$. WikiText extraction on private-finetuned Llama-3.2-1B/3B teachers shows the same fidelity/top-token tradeoff: \texttt{delta\_rank} lowers KL from CE-only $0.647{\to}0.296$ and $0.564{\to}0.336$, while sparse top-$K$ KD gives higher top-1 agreement but worse KL. A same-tokenizer Qwen3-1.7B$\to$0.6B stress test (\Cref{tab:qwen-cross-scale}) shows that top-$K$ methods no longer beat CE-only under capacity mismatch, while full-logit DKD still lowers KL ($0.559{\to}0.512$). Thus \Cref{tab:main}'s same-architecture PVR results should be read as an upper-bound-style extraction setting; cross-family and cross-scale PVR remain open.

$\Uk$ is a worst-case diameter, not typical tail mass: at $K{=}20$, Qwen3 has $\Uk=0.908$ while hidden tail mass averages $0.176$. A known base model can shrink the set (\Cref{thm:reference}): $\Ur=0.492$ for $\rho{=}1$, $0.573$ for $\rho{=}2$, and $0.752$ for $\rho{=}5$. Oracle calibration shows this prior is violated by aggressive SFT (median max perturbation 4.25; $\rho{=}5$ fails at 21\% of positions), so unrestricted $\Uk$ remains the relevant bound here.

\subsection{Interpretation}

At $K{=}20$, the certified binary lower bound is $R_{\mathrm{bin}}=0.538$ nats per token; for tolerance $\delta=0.1$, this lower bound already requires $\Uk\approx0.25$, far below observed values at $K\leq100$. For providers, top-$K$ censoring strongly limits prompt-only logit extraction of reasoning skills and partially limits code extraction, but generation/sequential access remains the dominant channel.

%% file: sections/6_conclusion.tex
\section{Conclusion}
\label{sec:conclusion}

We gave an exact identified-set characterization and certified KL recovery bounds for per-position distribution recovery under top-$K$ censored API access. The TV diameter is exactly $\Uk$; KL recovery admits a computable binary lower bound $R_{\mathrm{bin}}(\Uk)$ and an asymptotically matching $\Uk/e + O(\Uk^2)$ upper bound for small ambiguity. Thus the exact tight result is for TV diameter, while finite practical-$\Uk$ KL minimax recovery remains open. Empirically, access mode and teacher trajectories separate sharply: on GSM8K, full-logit DKD adds 45 PVR points over top-$K$ KD, and generation adds another 40. Generation is therefore a boundary condition on the threat model, not evidence that prompt-only top-$K$ geometry is irrelevant.

\paragraph{Limitations.}
The bounds are per-position: adaptive composition, autoregressive logit access, or generation can tighten or bypass the identified set. PVR is validated mainly within one model family, and API rounding, quantization, or temperature perturbation would change the geometry.

%% file: sections/A_appendix.tex
\section{Full Proofs}
\label{app:proofs}

\subsection{Proof of \Cref{thm:tv-diameter}}

\begin{proof}
Let $M = V - K$, $b = \exp(\tau)$, $\ZA = \sum_{v \in A} \exp(z_v)$, and $\alpha_v = \exp(z_v)/\ZA$.

For censored tokens $u \notin A$, write unnormalized weights $0 \leq y_u \leq b$. Let $Y = \sum_{u \notin A} y_u$ and $t = Y/(Z_A + Y)$. Then $t \in [0, \Uk]$ where $\Uk = Mb/(Z_A + Mb)$.

Every feasible distribution has $p_v = (1-t)\alpha_v$ for $v \in A$. For any two feasible distributions $\bp, \bq$ with tail masses $t, s$:
\begin{align*}
\TV(\bp, \bq) &= \frac{1}{2}\bigl(|t-s| + \|\bp_{\bar{A}} - \bq_{\bar{A}}\|_1\bigr) \\
&\leq \frac{1}{2}(|t-s| + t + s) = \max(t,s) \leq \Uk.
\end{align*}
Equality: take $\bp^{(0)} = (\alpha, \mathbf{0})$ and $\bp^{(\Uk)} = ((1{-}\Uk)\alpha, \frac{\Uk}{M}\mathbf{1})$. Then $\TV(\bp^{(0)}, \bp^{(\Uk)}) = \Uk$.
\end{proof}

\subsection{Proof of \Cref{thm:minimax}: Binary Endpoint Lower Bound}

\begin{proof}
Restrict $\bar{\cS}_K$ to $\{\bp^{(0)}, \bp^{(\Uk)}\}$. For any candidate $\bq$ with tail mass $s$:
\begin{align*}
\KL(\bp^{(0)} \| \bq) &= -\log(1-s), \\
\KL(\bp^{(\Uk)} \| \bq) &= (1{-}\Uk)\log\frac{1{-}\Uk}{1{-}s} + \Uk\log\frac{\Uk}{s}.
\end{align*}
Setting $\KL(\bp^{(0)} \| \bq) = \KL(\bp^{(\Uk)} \| \bq)$, we obtain
$\frac{s}{1-s} = \Uk(1{-}\Uk)^{(1-\Uk)/\Uk},$
giving $s^* = \Uk(1{-}\Uk)^{(1-\Uk)/\Uk}/(1 + \Uk(1{-}\Uk)^{(1-\Uk)/\Uk})$.

Expanding: let $h(\Uk) = 1 + \frac{1-\Uk}{\Uk}\log(1{-}\Uk) = \frac{\Uk}{2} + O(\Uk^2)$, so $s^* = \frac{\Uk}{e}e^{h(\Uk)}/(1 + \frac{\Uk}{e}e^{h(\Uk)}) = \Uk/e + O(\Uk^2)$.

The balanced risk is $R_{\mathrm{bin}} = -\log(1 - s^*) = \Uk/e + O(\Uk^2)$.
\end{proof}

\subsection{Proof of \Cref{thm:minimax}: Small-$\Uk$ Upper Bound}

\begin{proof}
Set $s = \Uk/e$ and define $q_v = (1-s)\alpha_v$ for $v \in A$, $q_u = s/M$ for $u \notin A$.

For any $\bp \in \bar{\cS}_K$ with tail mass $t$ and tail conditional $\mathbf{r}$ (where $r_u = p_u/t$):
\begin{equation*}
\KL(\bp \| \bq) = d(t \| s) + t \cdot \KL(\mathbf{r} \| \mathbf{u}_M),
\end{equation*}
where $d(t \| s) = t\log(t/s) + (1{-}t)\log((1{-}t)/(1{-}s))$ and $\mathbf{u}_M$ is uniform over $M$ tokens.

The per-token cap gives $r_u \leq \lambda(t)/M$ where $\lambda(t) = \Uk(1{-}t)/((1{-}\Uk)t)$. Therefore $\KL(\mathbf{r} \| \mathbf{u}_M) \leq \log \lambda(t)$.

Define $G_{\Uk}(t) = (1{-}t)\log\frac{1{-}t}{1{-}s} + t\log\frac{\Uk(1{-}t)}{(1{-}\Uk)s}$. Using $s = \Uk/e$ and $t + \log(1{-}t) \leq 0$:
\begin{equation*}
G_{\Uk}(t) \leq -(1{-}t)\log(1{-}s) - t\log(1{-}\Uk) \leq \Uk/e + O(\Uk^2)
\end{equation*}
uniformly for $t \in [0, \Uk]$ as $\Uk \to 0$. The binary lower bound has the same first-order expansion, giving asymptotic tightness in the small-ambiguity regime.
\end{proof}

\subsection{Proof of \Cref{thm:reference}}

\begin{proof}
The reference-aware identified set replaces the uniform cap $y_u \leq b$ with $y_u \leq B_u = \exp(\min(\tau, z_{\mathrm{ref}}(u) + \rho))$. The maximum tail mass becomes $t_{\max} = C_R/(Z_A + C_R)$ where $C_R = \sum_u B_u \leq Mb$, so $\Ur \leq \Uk$.

The extremal pair achieving $\TV = \Ur$ is $\bp^{(0)} = (\alpha, \mathbf{0})$ and $\bp^{(\Ur)}$ with tail $p_u = \Ur \cdot B_u/C_R$. The minimax analysis proceeds identically with $\Ur$ replacing $\Uk$ and optimal reserve $\beta_u = B_u/C_R$.
\end{proof}

\subsection{Non-Adaptive Product Composition}

\begin{proposition}[Non-adaptive product composition]
\label{prop:nonadaptive-composition}
Fix $m$ prompt-position queries before observing their top-$K$ outputs. Let $\bar{\cS}_{K,i}$ be the identified set induced by observation $i$, and assume no structural coupling across the $m$ conditional distributions beyond membership in the Cartesian product $\prod_{i=1}^m \bar{\cS}_{K,i}$. For estimators $\bq_{1:m}$ evaluated by average KL loss,
\begin{equation*}
\bar{R}_m^*
= \inf_{\bq_{1:m}} \sup_{\bp_i \in \bar{\cS}_{K,i}}
\frac{1}{m}\sum_{i=1}^m \KL(\bp_i \| \bq_i)
= \frac{1}{m}\sum_{i=1}^m R_{K,i}^* .
\end{equation*}
Consequently, $\bar{R}_m^* \geq m^{-1}\sum_i R_{\mathrm{bin}}(\Uk^{(i)})$, where $\Uk^{(i)}$ is the TV diameter for query $i$. If $\max_i \Uk^{(i)} \to 0$, coordinatewise symmetric estimators achieve
$m^{-1}\sum_i \Uk^{(i)}/e + O(m^{-1}\sum_i (\Uk^{(i)})^2)$.
\end{proposition}

\begin{proof}
After the $m$ non-adaptive observations are fixed, the estimator chooses one candidate distribution $\bq_i$ for each observed prompt-position pair. For any fixed $\bq_{1:m}$, the adversary's feasible set is a Cartesian product and the loss is separable, so
\begin{equation*}
\sup_{\bp_i \in \bar{\cS}_{K,i}}
\frac{1}{m}\sum_{i=1}^m \KL(\bp_i \| \bq_i)
= \frac{1}{m}\sum_{i=1}^m
\sup_{\bp_i \in \bar{\cS}_{K,i}} \KL(\bp_i \| \bq_i).
\end{equation*}
Taking the infimum over $\bq_{1:m}$ also separates across coordinates:
\begin{equation*}
\inf_{\bq_{1:m}} \frac{1}{m}\sum_{i=1}^m
\sup_{\bp_i \in \bar{\cS}_{K,i}} \KL(\bp_i \| \bq_i)
= \frac{1}{m}\sum_{i=1}^m
\inf_{\bq_i}\sup_{\bp_i \in \bar{\cS}_{K,i}} \KL(\bp_i \| \bq_i)
= \frac{1}{m}\sum_{i=1}^m R_{K,i}^* .
\end{equation*}
The lower bound follows by applying the binary-endpoint lower bound to each coordinate. The small-$\Uk^{(i)}$ upper bound follows by applying the coordinatewise symmetric estimator from the proof of \Cref{thm:minimax} to each observation and averaging the resulting risks. This proposition does not cover adaptive query selection, autoregressive query-and-append access, or parametric coupling induced by a shared LLM weight vector.
\end{proof}

\section{Explicit $O(\Uk^2)$ Bounds for the Binary Lower Bound}

For $0 < \Uk \leq 1/2$, the binary endpoint reserve satisfies:
\begin{equation*}
\left|s^* - \Uk/e\right| \leq \Uk^2.
\end{equation*}
The binary endpoint KL lower bound expands as:
\begin{equation*}
R_{\mathrm{bin}}(\Uk) = \frac{\Uk}{e} + \left(\frac{1}{2e} - \frac{1}{2e^2}\right)\Uk^2 + O(\Uk^3).
\end{equation*}

\paragraph{Numerical verification of the binary $O(\Uk^2)$ correction.}
\Cref{tab:gap} shows that $R_{\mathrm{bin}}(\Uk) - \Uk/e$ is well-described by the second-order expansion $(1/(2e) - 1/(2e^2))\Uk^2 \approx 0.116\,\Uk^2$ only for small $\Uk$. For the $\Uk$ values observed in practice ($\Uk > 0.8$), the correction is a substantial fraction of the first-order term (e.g., 62\% at $\Uk = 0.91$), confirming that the asymptotic $\Uk/e$ approximation degrades at large $\Uk$ and the exact $R_{\mathrm{bin}}$ formula should be used. We also report the maximum of the proof's finite-$\Uk$ upper envelope $G_{\Uk}$; its gap from $R_{\mathrm{bin}}$ is large at observed $\Uk$, so the paper does not claim exact finite-$\Uk$ minimax rates.

\begin{table}[h]
\centering
\caption{Second-order correction and finite-$\Uk$ conservatism. $G_{\max}$ is the maximum of the proof's finite-$\Uk$ upper envelope with reserve $s=\Uk/e$; its large gap at observed $\Uk$ motivates reporting $R_{\mathrm{bin}}$ only as a certified lower bound.}
\label{tab:gap}
\small
\begin{tabular}{@{}r cccc@{}}
\toprule
$\Uk$ & $R_{\mathrm{bin}}(\Uk)$ & $\Uk/e$ & $R_{\mathrm{bin}}{-}\Uk/e$ & $G_{\max}$ \\
\midrule
0.10 & 0.038 & 0.037 & 0.001 & 0.040 \\
0.30 & 0.123 & 0.110 & 0.012 & 0.142 \\
0.50 & 0.223 & 0.184 & 0.039 & 0.294 \\
0.70 & 0.349 & 0.258 & 0.092 & 0.559 \\
0.81 & 0.437 & 0.298 & 0.139 & 0.825 \\
0.91 & 0.541 & 0.335 & 0.206 & 1.309 \\
0.98 & 0.644 & 0.361 & 0.284 & 2.416 \\
\bottomrule
\end{tabular}
\end{table}

\section{Additional Experimental Details}

\paragraph{Math teacher training.} Qwen3-0.6B base model fine-tuned on GSM8K training split (7,473 examples) for 3,000 steps with AdamW (lr=$2{\times}10^{-5}$, weight decay=0.01, cosine schedule), batch size 4, max sequence length 512. Training format: ``Question: [Q]\textbackslash nAnswer: [A]'' where [A] includes the chain-of-thought solution with ``\#\#\#\# [answer]'' format.

\paragraph{GSM8K evaluation.} Greedy decoding (temperature=1, do\_sample=False), max 256 new tokens. Answer extracted via regex matching ``\#\#\#\# [number]'' or last number in generation. 500 test samples per evaluation.

\paragraph{On-task logit controls.} The on-task controls train on GSM8K prompts without teacher-generated trajectories. Full-logit DKD gives GSM8K/PVR/KL values: seed 42 = 27.4\% / 57.6\% / 0.0103, seed 123 = 27.4\% / 57.6\% / 0.0102, and seed 777 = 26.0\% / 54.2\% / 0.0102. On-task top-$K$ KD gives: seed 42 = 9.4\% / 13.3\% / 0.0167, seed 123 = 9.8\% / 14.3\% / 0.0162, and seed 777 = 7.2\% / 7.9\% / 0.0166.

\paragraph{Llama distribution-recovery stress test.} We fine-tune Llama-3.2-1B and Llama-3.2-3B on the WikiText-103 test split for 500 steps and train students on WikiText-103 train with $K{=}20$ top-logit access. These runs evaluate distribution recovery (KL and teacher top-1 agreement), not private-value recovery.

\begin{table}[h]
\centering
\caption{Cross-family distribution recovery on private-finetuned Llama teachers (mean $\pm$ population s.d.\ over 3 seeds). KL closure is measured relative to the CE-only student.}
\label{tab:llama-dist}
\small
\begin{tabular}{@{}l l c c c@{}}
\toprule
\textbf{Model} & \textbf{Method} & \textbf{KL}$\downarrow$ & \textbf{KL cl.} & \textbf{Top-1 agree}$\uparrow$ \\
\midrule
Llama-3.2-1B & \texttt{ce\_only} & $0.647_{\pm.004}$ & 0\% & $66.2_{\pm0.1}$ \\
 & \texttt{strict\_topk\_kd} & $1.065_{\pm.000}$ & $-$65\% & $80.4_{\pm0.1}$ \\
 & \texttt{delta\_rank} & $\mathbf{0.296}_{\pm.011}$ & \textbf{54\%} & $75.1_{\pm0.3}$ \\
 & \texttt{full\_logit} & $0.148_{\pm.000}$ & 77\% & $78.1_{\pm0.1}$ \\
\midrule
Llama-3.2-3B & \texttt{ce\_only} & $0.564_{\pm.002}$ & 0\% & $68.9_{\pm0.2}$ \\
 & \texttt{strict\_topk\_kd} & $1.036_{\pm.002}$ & $-$84\% & $77.1_{\pm0.0}$ \\
 & \texttt{delta\_rank} & $\mathbf{0.336}_{\pm.010}$ & \textbf{40\%} & $74.3_{\pm0.4}$ \\
 & \texttt{full\_logit} & $0.225_{\pm.001}$ & 60\% & $76.3_{\pm0.0}$ \\
\bottomrule
\end{tabular}
\end{table}

\paragraph{Cross-scale Qwen3 distribution-recovery stress test.}
We additionally ran a minimal same-tokenizer cross-scale check with a Qwen3-1.7B WikiText-private teacher and a Qwen3-0.6B student. This is a capacity-mismatch stress test, not a private-value evaluation: all methods train for 2000 steps on WikiText with $K{=}20$, and we report held-out distribution recovery over 3 seeds. The result in \Cref{tab:qwen-cross-scale} is deliberately conservative. Top-$K$ residual methods do not improve on the CE-only student, while full-logit DKD gives a small but stable KL reduction. We therefore do not use this run to claim cross-scale top-$K$ recovery; it supports the narrower point that full logits remain more informative than censored top-$K$ access under capacity mismatch.

\begin{table}[h]
\centering
\caption{Minimal cross-scale Qwen3 distribution-recovery stress test: Qwen3-1.7B teacher $\to$ Qwen3-0.6B student, WikiText, $K{=}20$, 2000 steps, mean $\pm$ population s.d.\ over 3 seeds. KL closure is measured relative to the CE-only student.}
\label{tab:qwen-cross-scale}
\small
\begin{tabular}{@{}l c c c@{}}
\toprule
\textbf{Method} & \textbf{KL}$\downarrow$ & \textbf{KL cl.} & \textbf{Top-1 agree}$\uparrow$ \\
\midrule
\texttt{ce\_only} & $0.559_{\pm.001}$ & 0\% & $67.92_{\pm0.03}$ \\
\texttt{strict\_topk\_kd} & $1.454_{\pm.001}$ & $-$160.0\% & $67.98_{\pm0.07}$ \\
\texttt{delta\_rank} & $0.597_{\pm.003}$ & $-$6.7\% & $67.54_{\pm0.05}$ \\
\texttt{dkd\_full} & $\mathbf{0.512}_{\pm.000}$ & \textbf{8.4\%} & $67.97_{\pm0.08}$ \\
\bottomrule
\end{tabular}
\end{table}

This stress test also clarifies the scope of \Cref{tab:main}: same-architecture, same-scale extraction is an informative controlled setting and likely an upper bound on practical prompt-only extraction efficacy. Cross-scale and cross-architecture private-value recovery remain important open evaluations.

\paragraph{Code teacher training.} Qwen3-0.6B base model fine-tuned on MBPP training split (374 examples) for 5 epochs with AdamW (lr=$2{\times}10^{-5}$, weight decay=0.01, cosine schedule), batch size 4, max sequence length 512, bfloat16 precision.

\paragraph{MBPP evaluation.} Execution-based: greedy decoding (max 256 new tokens), generated code is executed against 3 held-out assertions per problem using a 1-second timeout per assertion. A problem is correct only if all scored assertions pass. We evaluate 200 test samples and 3 random seeds. The supplemental script \texttt{scripts/run\_code\_domain\_v2.py} generates the execution-pass fields reported here; \texttt{scripts/run\_code\_domain.py} is an older name-match diagnostic. The prompt includes \texttt{test\_list[0]} as a public specification hint, following the MBPP convention of providing task-level examples~\citep{austin2021program}; scoring excludes this prompted assertion and uses held-out assertions from \texttt{test\_list[1:4]}. The reported MBPP run uses $K{=}20$, 800 student steps, 25 held-out WikiText texts for post-hoc KL, all 374 MBPP training examples for generation SFT, and batch size 8 for execution evaluation.

\begin{figure}[h]
\centering
\includegraphics[width=0.85\textwidth]{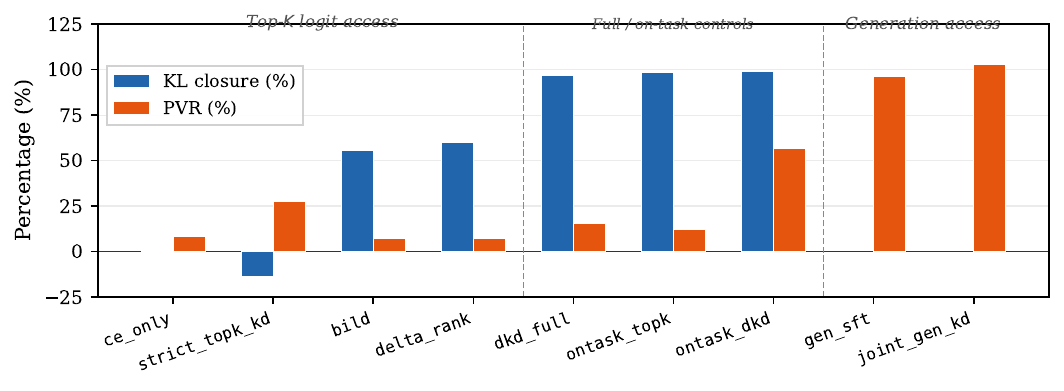}
\caption{KL closure vs.\ PVR across extraction methods, including on-task logit controls.}
\label{fig:kl-vs-pvr}
\end{figure}

%% file: references.bib
@inproceedings{carlini2024stealing,
  title={Stealing Part of a Production Language Model},
  author={Carlini, Nicholas and Paleka, Daniel and Dvijotham, Krishnamurthy Dj and Steinke, Thomas and Hayase, Jonathan and Cooper, A. Feder and Lee, Katherine and Jagielski, Matthew and Nasr, Milad and Conmy, Arthur and others},
  booktitle={International Conference on Machine Learning (ICML)},
  year={2024}
}

@article{golowich2025provably,
  title={Provably Learning from Modern Language Models via Low Logit Rank},
  author={Golowich, Noah and Liu, Allen and Shetty, Abhishek},
  journal={arXiv preprint arXiv:2512.09892},
  year={2025}
}

@article{golowich2025sequences,
  title={Sequences of Logits Reveal the Low Rank Structure of Language Models},
  author={Golowich, Noah and Liu, Allen and Shetty, Abhishek},
  journal={arXiv preprint arXiv:2510.24966},
  year={2025},
  eprint={2510.24966},
  archiveprefix={arXiv},
  doi={10.48550/arXiv.2510.24966}
}

@book{manski2003partial,
  title={Partial Identification of Probability Distributions},
  author={Manski, Charles F.},
  year={2003},
  publisher={Springer}
}

@incollection{molinari2020microeconometrics,
  title={Microeconometrics with Partial Identification},
  author={Molinari, Francesca},
  booktitle={Handbook of Econometrics},
  volume={7},
  pages={355--486},
  year={2020},
  publisher={Elsevier},
  doi={10.1016/bs.hoe.2020.05.002}
}

@inproceedings{tramer2016stealing,
  title={Stealing Machine Learning Models via Prediction APIs},
  author={Tram{\`e}r, Florian and Zhang, Fan and Juels, Ari and Reiter, Michael K. and Ristenpart, Thomas},
  booktitle={USENIX Security Symposium},
  year={2016}
}

@inproceedings{jagielski2020high,
  title={High Accuracy and High Fidelity Extraction of Neural Networks},
  author={Jagielski, Matthew and Carlini, Nicholas and Berthelot, David and Kurakin, Alex and Papernot, Nicolas},
  booktitle={USENIX Security Symposium},
  year={2020}
}

@article{finlayson2024logits,
  title={Logits of API-Protected LLMs Leak Proprietary Information},
  author={Finlayson, Matthew and Ren, Xiang and Swayamdipta, Swabha},
  journal={arXiv preprint arXiv:2403.09539},
  year={2024}
}

@inproceedings{agarwal2024gkd,
  title={On-Policy Distillation of Language Models: Learning from Self-Generated Mistakes},
  author={Agarwal, Rishabh and Vieillard, Nino and Zhou, Yongchao and Stanczyk, Piotr and Ramos Garea, Sabela and Geist, Matthieu and Bachem, Olivier},
  booktitle={International Conference on Learning Representations (ICLR)},
  year={2024}
}

@inproceedings{gu2024minillm,
  title={MiniLLM: On-Policy Distillation of Large Language Models},
  author={Gu, Yuxian and Dong, Li and Wei, Furu and Huang, Minlie},
  booktitle={International Conference on Learning Representations (ICLR)},
  year={2024},
  eprint={2306.08543},
  archiveprefix={arXiv}
}

@article{hinton2015distilling,
  title={Distilling the Knowledge in a Neural Network},
  author={Hinton, Geoffrey and Vinyals, Oriol and Dean, Jeff},
  journal={arXiv preprint arXiv:1503.02531},
  year={2015}
}

@article{deepseek2025r1,
  title={DeepSeek-R1 incentivizes reasoning in LLMs through reinforcement learning},
  author={Guo, Daya and Yang, Dejian and Zhang, Haowei and others},
  journal={Nature},
  volume={645},
  pages={633--638},
  year={2025},
  doi={10.1038/s41586-025-09422-z}
}

@inproceedings{hewitt2022truncation,
  title={Truncation Sampling as Language Model Desmoothing},
  author={Hewitt, John and Manning, Christopher D. and Liang, Percy},
  booktitle={Findings of EMNLP},
  year={2022}
}

@article{austin2021program,
  title={Program Synthesis with Large Language Models},
  author={Austin, Jacob and Odena, Augustus and Nye, Maxwell and Bosma, Maarten and Michalewski, Henryk and Dohan, David and Jiang, Ellen and Cai, Carrie and Terry, Michael and Le, Quoc and Sutton, Charles},
  journal={arXiv preprint arXiv:2108.07732},
  year={2021},
  eprint={2108.07732},
  archiveprefix={arXiv}
}

@article{topk_attention_tv2025,
  title={A Mathematical Theory of Top-k Sparse Attention via Total Variation Distance},
  author={Tzachristas, Georgios and Deng, Lei and Tzachristas, Ioannis and Zhang, Gong and Chen, Renhai},
  journal={arXiv preprint arXiv:2512.07647},
  year={2025}
}
